\documentclass[runningheads]{article}
\usepackage{graphicx}
\graphicspath{{fig/}}
\DeclareGraphicsExtensions{.pdf,.png,.jpg}

\PassOptionsToPackage{dvipsnames,usenames}{xcolor}
\usepackage{tikz}
\usetikzlibrary{decorations.pathreplacing}
\usepackage{subfigure}

\usepackage[affil-it]{authblk}
\usepackage{tabularx}
\usepackage{booktabs}
\usepackage[implicit=true,unicode=true,bookmarks=true,bookmarksnumbered=true,bookmarksopen=false,colorlinks=true,urlcolor=ForestGreen,citecolor=Blue,linkcolor=BrickRed]{hyperref}
\hypersetup{
  pdftitle={Filtering rules for flow time minimization in a Parallel Machine Scheduling Problem},
  pdfkeywords={Parallel Machine  Scheduling;  Job families; Flow time;  Machine  Disqualification;  Filtering  Algorithm;  Cost-Based Filtering.},
  pdfpagelayout=OneColumn,
  pdfnewwindow=true
}
\usepackage{caption}
\usepackage{amsmath,amsthm}
\usepackage{mathtools}
\usepackage[tworuled,vlined,nokwfunc,titlenotnumbered]{algorithm2e}
\usepackage{array}

\usepackage{booktabs}
\newcolumntype{t}{>{\bfseries\ttfamily}c}
\usepackage[numbers,sort&compress]{natbib}
\usepackage{xspace}
\usepackage{paralist}
\usepackage{wrapfig}

\allowdisplaybreaks[2]

\newcommand{\N}{{\cal N}}
\newcommand{\M}{{\cal M}}

\newcommand{\F}{{\cal F}}
\renewcommand{\nl}{\tabularnewline} 
\newcolumntype{C}{>{\centering}X}
\newcolumntype{P}[1]{>{\centering\arraybackslash}p{#1}}
\newcolumntype{M}{>{$}c<{$}}
\newcolumntype{L}{>{$}l<{$}}
\newcolumntype{T}{>{$}X<{$}}

\newcommand{\IP}{ILP model\xspace}
\newcommand{\CP}{CP$_{O}$ model\xspace}
\newcommand{\CPN}{CP$_{N}$ model\xspace}

\newcommand{\setARE}{\N_{m}^A}
\newcommand{\setExtFT}{\N_m^O}
\newcommand{\setProofEFT}{\overline{N_m^O}}

\newcommand{\setNJfm}{\N_m^f}

\newcommand{\setNJm}{\N_m^{NS}}

\newcommand{\Seq}{\texttt{sequencing}\xspace}

\newtheorem{filter}{Rule}
\newcommand{\refR}[1]{Rule~\ref{filt:#1}\xspace}
 \newtheorem{example}{Example}
 \newtheorem{theorem}{Theorem}
 \newtheorem{definition}{Definition}
\begin{document}

\title{Filtering rules for flow time minimization in a Parallel Machine Scheduling Problem} 
\author{Margaux Nattaf}
\affil{Univ. Grenoble Alpes, CNRS, Grenoble INP, G-SCOP, 38000 Grenoble, France\\ 
\texttt{margaux.nattaf@grenoble-inp.fr}}
\author{Arnaud Malapert}
\affil{ Universit\'e C\^ote d'Azur, CNRS, I3S, France \\
 \texttt{arnaud.malapert@unice.fr}}

\setcounter{tocdepth}{3}

\maketitle

\begin{abstract}
  This paper studies  the scheduling of jobs of  different families on
  parallel  machines with  qualification constraints.   Originating from
  semi-conductor manufacturing, this constraint imposes a time threshold
  between the execution of two jobs  of the same family.  Otherwise, the
  machine becomes disqualified for this family.  The goal is to minimize
  both the flow time and the number of disqualifications.
  Recently, an  efficient constraint  programming model  has been
  proposed.  However, when priority is given to the flow time objective,
  the efficiency of the model can be improved.
  
  This paper uses a polynomial-time  algorithm which minimize the flow
  time for a single machine relaxation where disqualifications are not
  considered.   Using this  algorithm  one can  derived filtering
  rules on  different variables of  the model. Experimental  results are
  presented showing  the effectiveness  of these  rules.
  They  improve  the competitiveness  with  the  mixed integer  linear
  program of the literature.
  
 \textbf{keywords : }Parallel Machine  Scheduling,  Job families , Flow
    time ,  Machine  Disqualification,  Filtering  Algorithm ,  Cost-Based Filtering.
\end{abstract}

\section{Introduction}
This  paper considers  the  scheduling of  job  families on  parallel
machines  with time  constraints  on machine  qualifications. In  this
problem,  each job  belongs  to a  family  and a  family  can only  be
executed on a subset of  qualified machines. In addition, machines can
lose their qualifications during the schedule.  Indeed, if no job of a
family is  scheduled on a machine  during a given amount  of time, the
machine  lose  its qualification  for  this  family.  The goal  is  to
minimize the  sum of job completion  times, i.e. the flow  time, while
maximizing the number of qualifications at the end of the schedule.

This problem,  called scheduling Problem with  Time Constraints (PTC),
is introduced  in~\cite{Obeid2014}.  It  comes from  the semiconductor
industries. Its goal is to  introduce constraints coming from Advanced
Process Control (APC) into a scheduling problem.  APC's systems are used
to control processes and equipment  to reduce variability and increase
equipment efficiency.  In PTC, qualification constraints and 
objective come from APC and more precisely from what is called Run to
Run control.  More details about the industrial problem can  be found
in~\cite{Nattaf2018C}.

Several      solution     methods       has      been     defined  for
PTC~\cite{Obeid2014,Nattaf2018C,Nattaf2019C}.    In  particular,   the
authors of~\cite{Nattaf2019C} present two pre-existing models: a Mixed
Integer  Linear  Program  (MILP)  and a  Constraint  Programming  (CP)
model. Furthermore, they define a new CP model taking advantage 
of advanced CP features to model machine disqualifications.  However, the
paper shows  that when the  priority objective  is the flow  time, the
performance of the CP model can be improved. 

The objective  of this  paper is  to improve  the performances  of the
CP model for the flow time objective.  To do so, a relaxed version of PTC
where  qualification  constraints are  removed is considered. For this
relaxation, the  results of~\citet{Mason1991} are adapted  to define an
algorithm  to optimally  sequence jobs  on one  machine in  polynomial
time.   This  algorithm  is  then used  to  define  several  filtering
algorithms for PTC.  

Although,  the  main  result  of this  paper  concerns  the  filtering
algorithms for PTC, there is also two more general results incident to
this work.   First, those  algorithms can be  directly applied  to any
problem having  a flow time  objective  and which  can be relaxed to a
parallel machine  scheduling problem with  sequence-independent family
setup times.  Secondly,  the approach is related  to cost-based domain
filtering~\cite{foccaci.ea-99}, a general approach to define filtering
algorithms for optimization problems. 

The paper is organized as follows. Section~\ref{sec:pb} gives a formal
description of the problem the CP model for PTC.
Section~\ref{sec:relaxation}  presents  the  relaxed problem  and  the
optimal machine flow time computation of the relaxed problem. 
Section~\ref{sec:filter} shows how this flow time is used to define
filtering     rules    and     algorithms    for     PTC.     Finally,
Section~\ref{sec:expe}   shows  the   performance  of   the  filtering
algorithms and compares our results to the literature. 

\section{Problem description and modeling}
\label{sec:pb}

In this section,  a formal description of PTC is  given.  Then, a part
of the CP  model of~\cite{Nattaf2019C} is presented.  The  part of the
model presented is  the part that is useful to  present our cost based
filtering   rules   and   correspond    to   the   modeling   of   the
relaxation. Indeed, as  we are interested in the  flow time objective,
the machine qualification modeling is not presented in this paper.

\subsection{PTC description}

Formally, the problem takes as input a set of jobs,
$\N=\{1,\dots,N\}$, a set of families $\F =\{1,\dots, F\}$ and a set
of machines, $\M=\{1,\dots,M\}$. Each job $j$ belongs to a family and the
family associated with $j$ is denoted by $f(j)$. For each family $f$,
only a subset of the machines $\M_f \subseteq \M$, is able to process
a job of $f$. A machine $m$ is said to be qualified to process a family
$f$ if $m \in \M_f$.
\noindent
Each  family $f$ is associated with the following parameters:

\begin{compactitem}
\item $n_f$ denotes the number of jobs in the family. Note that
  $\sum_{f \in\F} n_f = N$.
\item $p_f$ corresponds to the processing time of jobs in $f$.
\item $s_f$ is the setup time required to switch
  the production  from a job  belonging to a family  $f' \neq f$  to the
  execution of a job of $f$. Note that this setup time is independent of
  $f'$, so it is sequence -independent.  In addition, no setup time is
  required neither between the execution of two jobs of the same family
  nor at the beginning of the schedule, i.e. at time $0$.
\item $\gamma_f$ is the threshold value for the time interval between
  the execution of two jobs of $f$ on the same machine. Note that this
  time interval is computed on a start-to-start basis, i.e. the
  threshold is counted from the start of a job of family $f$ to the
  start of the next job of $f$ on machine $m$. Then, if there is a time
  interval $]t,t +\gamma_f]$ without any job of $f$ on a machine, the
  machine lose its qualification for $f$. 
\end{compactitem}

The objective is to minimize both the sum of job completion times,
i.e. the flow time, and the number of qualification looses or
disqualifications.  An  example  of PTC  together  with  two  feasible
solutions is now presented.

\begin{example}
  \label{ex:PTC}
  Consider  the  instance  with  $N=10,\   M=2$  and  $F=3$  given  in
  Table~\ref{subfig:exOptInst}. 
  Figure~\ref{fig:exPTC} shows two feasible solutions.
  The first solution, described by Figure~\ref{subfig:exOptCf}, is optimal in
  terms of flow time. For this solution, the flow time is equal to
  $1+2+9+15+21+1+2+12+21+30 = 114$ and the number of qualification
  losses is $3$. Indeed, machine $1$ ($m_1$) loses its qualification for
  $f_3$ at time $22$ since there is no job of $f_3$ starting in interval
  $]1,22]$ which is of size $\gamma_3 = 21$. The same goes for $m_2$ and
  $f_3$ at time $22$ and for $m_2$ and $f_2$ at time $26$.
  \\
  The second solution, described by Figure~\ref{subfig:exOptDisq}, is
  optimal in terms of number of disqualifications. Indeed, in this
  solution, none of the machines loses their qualifications. However,
  the flow time is equal to $1+2+9+17+19+9+18+20+27+37=159$. 

  \begin{figure}
    \centering
    \subfigure[Instance with $N=10,\ M=2$ and $F=3$\label{subfig:exOptInst}]{
       \begin{tabularx}{0.26\linewidth}[b]{c|CCCCC}
        \toprule
       $f$ & $n_f$ & $p_f$ & $s_f$ & $\gamma_f$ & $ \M_f$\nl
        \midrule
        1&3&9&1&25&$\{2\}$\nl
        2&3&6&1&26&$\{1,2\}$\nl
        3&4&1&1&21&$\{1,2\}$\nl
        \bottomrule
      \end{tabularx}
    }
    \subfigure[An optimal solution for the flow time objective \label{subfig:exOptCf}]{
      \begin{tikzpicture}[xscale=0.25]
        \node (O) at (0,0) {};
        \draw[->] (O.center) -- ( 31,0); 
        \draw[->] (O.center) -- (0,1.25); 
        \draw (0, 0.25) node[left] {$m_1$};
        \draw (0, 0.75) node[left] {$m_2$};
        \draw (2,0) -- (2,-0.1) node[below] {$2$};
        \draw (9,0) -- (9,-0.1) node[below] {$9$};
        \draw (12,0) -- (12,-0.1) node[below] {$12$};
        \draw (15,0) -- (15,-0.1) node[below] {$15$};
        \draw (21,0) -- (21,-0.1) node[below] {$21$};
        \draw (22,0) -- (22,-0.1) node[below] {\scriptsize $22$};
        \draw (26,0) -- (26,-0.1) node[below] {\scriptsize $26$};
        \draw (30,0) -- (30,-0.1) node[below] {$30$};
        \draw[fill= red!90!black!40!]  (3,0.5) rectangle (12,1)
        node[midway] {$f_1$}; 
        \draw[fill= red!90!black!40!]  (12,0.5) rectangle (21,1)
        node[midway] {$f_1$}; 
        \draw[fill= red!90!black!40!]    (21,0.5) rectangle (30,1)
        node[midway] {$f_1$}; 
        \draw[fill=blue!40!]  (3,0) rectangle (9,0.5)
        node[midway] {$f_2$}; 
        \draw[fill=blue!40!]  (9,0) rectangle (15,0.5)
        node[midway] {$f_2$}; 
        \draw[fill=blue!40!]  (15,0) rectangle (21,0.5)
        node[midway] {$f_2$}; 
        \draw[fill=green!80!black!40!] (0,0) rectangle (1,0.5) node[midway,black] {$f_3$};
        \draw[fill=green!80!black!40!]  (1,0) rectangle (2,0.5) node[midway] {$f_3$};
        \draw[fill=green!80!black!40!]  (0,0.5) rectangle (1,1) node[midway] {$f_3$};
        \draw[fill=green!80!black!40!]  (1,0.5) rectangle (2,1) node[midway] {$f_3$};
      \end{tikzpicture}
    }
    
    \subfigure[An optimal solution for qualification losses \label{subfig:exOptDisq}]{
      \begin{tikzpicture}[xscale=0.25]
        \node (O) at (0,0) {};
        \draw[->] (O.center) -- ( 38,0);
        \draw[->] (O.center) -- (0, 1.25);
        \draw (0, 0.25) node[left] {$m_1$};
        \draw (0, 0.75) node[left] {$m_2$};
        
        \draw (2,0) -- (2,-0.1) node[below] {$2$};
        \draw (3,0) -- (3,-0.1) node[below] {$3$};
        \draw (9,0) -- (9,-0.1) node[below] {$9$};
        \draw (11,0) -- (11,-0.1) node[below] {$11$};
        \draw (17,0) -- (17,-0.1) node[below] {$17$};
        \draw (20,0) -- (20,-0.1) node[below] {$20$};
        \draw (27,0) -- (27,-0.1) node[below] {$27$};
        \draw (37,0) -- (37,-0.1) node[below] {$37$};
        \draw[fill= red!90!black!40!]  (0,0.5) rectangle (9,1)
        node[midway] {$f_1$}; 
        \draw[fill= red!90!black!40!]  (9,0.5) rectangle (18,1)
        node[midway] {$f_1$}; 
        \draw[fill= red!90!black!40!]  (28,0.5) rectangle (37,1)
        node[midway] {$f_1$}; 
        \draw[fill=blue!40!]  (3,0) rectangle (9,0.5)
        node[midway] {$f_2$}; 
        \draw[fill=blue!40!]  (11,0) rectangle (17,0.5)
        node[midway] {$f_2$}; 
        \draw[fill=blue!40!] (21,0.5) rectangle (27,1)
        node[midway] {$f_2$}; 
        \draw[fill=green!80!black!40!]  (0,0) rectangle (1,0.5) node[midway] {$f_3$};
        \draw[fill=green!80!black!40!]  (1,0) rectangle (2,0.5) node[midway] {$f_3$};
        \draw[fill=green!80!black!40!]  (18,0) rectangle (19,0.5) node[midway] {$f_3$};
        \draw[fill=green!80!black!40!]  (19,0.5) rectangle (20,1) node[midway] {$f_3$};
      \end{tikzpicture}
    }
    \caption{Two solution examples for PTC.}
    \label{fig:exPTC}
    
  \end{figure}
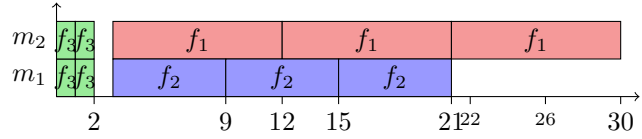
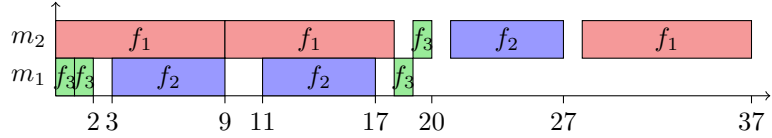
\end{example}

\subsection{CP model}
\label{sec:CP}

In   the    following   section,   the   part   of   the    CP   model
of~\cite{Nattaf2019C}    which       is    useful   for  this  work is
recalled.   This part  corresponds  to a  Parallel Machine  Scheduling
Problem (PMSP) with  family setup times. New  auxiliary variables used
by  our  cost  based  filtering  rules  are  also  introduced.   These
variables are written in bold in the variable description. 
To  model  the  PMSP  with family  setup  times,  (optional)  interval
variables  are  used~\cite{Laborie08,Laborie09}.    To  each  interval
variable $J$,  a start time $st(J)$,  an end time $et(J)$,  a duration
$d(J)$ and  an execution status  $x(J)$ is associated.   The execution
status $x(J)$ is equal to $1$ if and only if $J$ is present in the solution and
$0$ otherwise.
\\
The following set of variables is used:
\begin{compactitem}
\item  $jobs_j,\ \forall  j \in  \N$: Interval  variable modeling  the
  execution of job $j$; 
\item $altJ_{j,m},\  \forall (j,m) \in \N  \times \M_{f(j)}$: Optional
  interval variable  modeling  the assignment of
  job $j$ to machine $m$;
\item   $\mathbf{flowtime_m}$  and   $flowtime$:  Integer   variables  modeling
  respectively the flow time on machine $m$ and the global flow time;
\item  $\mathbf{nbJobs_{f,m}},  \forall  (f,m)  \in \F  \times  \M_f$:  Integer
  variable modeling the number of jobs of family $f$ scheduled on $m$;
\item  $\mathbf{nbJobs_{m}},  \forall  m  \in  \M$:  Integer  variable
  modeling the number of jobs scheduled on $m$.
\end{compactitem}
To  model  the  PMSP  with  setup time,  the  following  sets  of
constraints is used:
\footnotesize 
\begin{align}
  & flowTime = \sum_{j \in \N} et(jobs_j) & & \label{eq:objFT}\\
  &alternative \left(jobs_j, \left\{ altJ_{j,m} | m \in \M_{f(j)} \right\}
    \right)&  & \forall j \in \N \label{eq:alternative}\\
  &noOverlap  \left(  \left\{  altJ_{j,m}   |\forall  j\  s.t.\  m  \in
    \M_{f(j)} \right\}, S \right) 
                                          & &\forall m \in \M \label{eq:noOverlap}\\
  & flowtime = \sum_{m \in \M} flowtime_m & & \label{eq:sumFTm}\\ 
  & flowtime_m  = \sum_{j  \in \N}  et(altJ_{j,m}) &  & \forall  m \in
                                                        \M \label{eq:sumET}\\ 
  & nbJobs_{f,m} = \sum_{j \in \N; f(j) = f} x(altJ_{j,m}) & & \forall 
                                                               (f,m)    \in
                                                               \F
                                                               \times
                                                               \M_f \label{eq:nbJobsfm1}\\
  &  \sum_{m   \in  \M}  nbJobs_{f,m}  =   n_f  &  &  \forall   f  \in
                                                     \F \label{eq:nbJobsfm2}\\ 
  &  \sum_{ f  \in \F}  nbJobs_{f,m} =  nbJobs_{m} &  & \forall  m \in
                                                        \M \label{eq:nbJobsm1}\\ 
  & \sum_{m \in \M} nbJobs_m = N & & \label{eq:nbJobsm2}
\end{align}
\normalsize

Constraint~\eqref{eq:objFT} is used to compute  the flow time of the
schedule. Constraints~\eqref{eq:alternative}--\eqref{eq:noOverlap} are
used     to      model     the     PMSP     with      family     setup
time. Constraints~\eqref{eq:alternative} model  the assignment of jobs
to machine.  Constraints~\eqref{eq:noOverlap} ensure that jobs  do not
overlap and enforce setup times.  Note that $S$ denotes the setup time
matrix:  $(S_{f,f'})$  is equal  to  $0$  if $f  =  f'$  and to  $s_f$
otherwise.  A  complete  description  of {\it  alternative}  and  {\it
  noOverlap} constraints can be found in~\cite{Laborie08,Laborie09}.
\\
In~\cite{Nattaf2019C}, additional  constraints are used to  make the
model    stronger,    e.g.    ordering    constraints,    cumulative
relaxation. They are not presented in  this paper. 
\\
Constraints~\eqref{eq:sumFTm}--\eqref{eq:nbJobsm2}  are used  to link
the  new  variables   to  the  model.  Constraints~\eqref{eq:sumFTm}
and~\eqref{eq:sumET}       ensure       machine      flow       time
computation. Constraints~\eqref{eq:nbJobsfm1} compute  the number of
jobs      of      family      $f$     executed      on      machine
$m$. Constraints~\eqref{eq:nbJobsfm2} make sure  the right number of
jobs  of  family  $f$ is  executed.  Constraints~\eqref{eq:nbJobsm1}
and~\eqref{eq:nbJobsm2}           are          equivalent           to
constraints~\eqref{eq:nbJobsfm1} and~\eqref{eq:nbJobsfm2}  but for the
total number of jobs scheduled on machine $m$.
The bi-objective optimization is  a lexicographical optimization or its
linearization \cite{ehrgott-00}. 

\section{Relaxation description and sequencing} 
\label{sec:relaxation}

\subsection{${\cal R}$-PTC description}

The  relaxation  of  PTC  (${\cal   R}$-PTC)  is  a  parallel  machine
scheduling problem with sequence-independent family setup time without
the qualification constraints (parameter $\gamma_f$). 
The objective is then to minimize the flow time. 
In this section, it is assumed that a total the assignment of jobs to
machines is already done and the  objective is to sequence jobs so the
flow time is  minimal. Therefore, since the sequencing of  jobs on $M$
machines  can  be seen  as  $M$  one  machine problems,  this  section
presents how jobs can be sequenced optimally on one machine.  
In  Section~\ref{sec:filter}, the  cost-based  filtering rules  handle
partial assignments of jobs to the machines. 

\subsection{Optimal sequencing for ${\cal R}$-PTC}

The  results   presented  in   this  section  were  first  described
in~\cite{PMS2020}.
They are adapted from~\citet{Mason1991} who considers an initial setup
at the beginning of the schedule. 
The results are just summarized in this paper.

First, a solution can be represented as a sequence $S$ representing an
ordered   set  of  $n$  jobs.  Considering  job  families  instead  of
individual jobs, $S$ can be seen as  a series of blocks, where a block
is a  maximal consecutive sub-sequence  of jobs  in $S$ from  the same
family (see Figure~\ref{fig:block}).  Let $B_i$ be the $i$-th block of
the sequence,  $S =  \{B_1,B_2,\dots,B_r\}$. Hence,  successive blocks
contain  jobs from  different families.   Therefore, there  will be  a
setup time before each block (except the first one).

\begin{figure}[!htb]
  \begin{center}
    \begin{tikzpicture}[xscale=0.25]
      \node (O) at (-0.5,0) {};
      \draw[->] (O.center) -- ( 38,0);
      
      \draw[fill= red!90!black!40!]  (0,0.5) rectangle (9,0)
      node[midway] {$f_1$}; 
      \draw[fill= red!90!black!40!]  (9,0.5) rectangle (18,0)
      node[midway] {$f_1$}; 
      \draw[fill= red!90!black!40!]  (28,0.5) rectangle (37,0)
      node[midway] {$f_1$}; 
      \draw[fill=blue!40!] (21,0.5) rectangle (27,0)
      node[midway] {$f_2$}; 
      \draw[fill=green!80!black!40!]    (19,0.5)  rectangle   (20,0)
      node[midway] {$f_3$};

      \foreach \i in {0,5,10,...,35}
      { \draw (\i,0) -- (\i,-0.1) node[below] {\scriptsize $\i$}; }

      \begin{scope}[yshift = -1.5cm]
        \node (O) at (-0.5,0) {};
        \draw[->] (O.center) -- ( 38,0);
        
        \draw[fill= red!90!black!40!]  (0,0.5) rectangle (18,0)
        node[midway] {$B_1$};
        \draw[fill= red!90!black!40!]  (28,0.5) rectangle (37,0)
        node[midway] {$B_4$}; 
        \draw[fill=blue!40!] (21,0.5) rectangle (27,0)
        node[midway] {$B_3$}; 
        \draw[fill=green!80!black!40!]    (19,0.5)  rectangle   (20,0)
        node[midway] {$B_2$};

        \foreach \i in {0,5,10,...,35}
        { \draw (\i,0) -- (\i,-0.1) node[below] {\scriptsize $\i$}; }
      \end{scope}
    \end{tikzpicture}
    \caption{Block representation of a solution.}
    \label{fig:block}
  \end{center}
\end{figure}
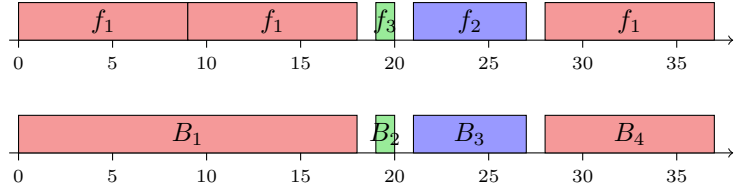

The idea of the algorithm is to adapt the Shortest Processing Time ($SPT$) 
rule~\cite{Smith} for blocks instead of individual jobs.  To this end,
blocks are  considered as individual  jobs with processing  time $P_i=
s_{f_i} +  |B_i|\cdot p_{f_i}  $ and weight  $W_i= |B_i|$  where $f_i$
denotes the family of jobs in $B_i$ (which is the same for all jobs in
$B_i$).

The first theorem  of this section states that there  always exists an
optimal solution  $S$ containing exactly  $|\F|$ blocks and  that each
block $B_i$ contains  all jobs of the family $f_i$.  That is, all jobs
of a family are scheduled consecutively. 

\begin{theorem}
  \label{th4}
  Let ${\cal I}$ be an instance of the problem. There exists an
  optimal  solution $S^*=  \{B_1,\dots, B_{|\F|}\}$  such that  $|B_i| =
  n_{f_i}$ where $f_i$ is the family of jobs in $B_i$.
\end{theorem}

\begin{proof}[Sketch]
  For a  complete proof  of the theorem, see~\cite{PMS2020}. \\
  Consider an  optimal solution  $S=  \{ B_1,  \dots, B_u, \dots, B_v,
  \dots, B_r\}$ with two blocks $B_u$ and $B_v$ ($u<v$), containing jobs
  of the same family $f_u = f_v = f $. Then, moving the first job of
  $B_v$ at the end of block $B_u$ can only improve the solution.

  Indeed, let us define $P$ and $W$ as: $P = \sum_{i = u+1}^{v-1} P_i + s_f$
  and $W =  \sum_{i=u+1}^{v-1} |B_i|$. In addition, let  $S'$ be the
  sequence formed by moving the first job of $B_v$, say 
  job $j_v$,  at the end of  block $B_u$.  The difference  on the flow
  time between $S$ and $S'$, is as follows: 

  $FT_{S'} - FT_S = \left\{
    \begin{array}{lcl}
      W\cdot p_f - P & & \text{ if } |B_v| = 1\\
      W\cdot p_f - P - \sum_{i = v+1}^r |B_i| \cdot s_f & 
                       & \text{ if }  |B_v|  > 1\\
    \end{array}
  \right.
  $

  \medskip

  Using Lemma 1 of~\cite{PMS2020} stating that $P/W \ge p_f$, then
  $FT_{S'} -  FT_{S} < 0  $ and the flow  time is improved  in $S'$.
  Hence, moving  the first job  of $B_v$ at  the end of  block $B_u$
  leads to a solution $S'$ at least as good as $S$.
  
  Therefore,  a   block  $B_i$   contains   all   jobs  of   family
  $f_i$. Indeed, if not, applying  the previous operation leads to a
  better solution.  Hence, $|B_i |  = n_{f_i}$ and there are exactly
  $F$ blocks in the optimal solution, i.e. one block per family. 
\end{proof}
At this point, the number of  block and theirs contents are defined.
The next step is to order them. To this end, the concept of weighted
processing time is also adapted to blocks as follows.

\begin{definition}
  The  Mean Processing  Time  (MPT) of  a block  $B_i$  is defined  as
  $MPT(B_i) = P_i/W_i$.
\end{definition}

One  may think  that,  in an  optimal solution,  jobs  are ordered  by
$SMPT$ (Shortest Mean  Processing Time).  However, this  is not always
true  since  no setup  time  is  required  at  time $0$.  Indeed,  the
definition  of block  processing time  always considers  a
setup time  before the block.  In our case, this  is not true  for the
first block.  Example~\ref{ex:CEMason} gives a  counter-example showing
that scheduling all blocks according to the SMPT rule is not optimal.
  
  \begin{example}[Counter-example -- Figure~\ref{fig:CEMason}]
    \label{ex:CEMason}
    Consider the instance given by Table~\ref{subfig:CEMasonInst} that
    also gives  the MPT  of each family  ($MPT_f = p_f  + (s_f  \div \
    n_f)$).  Figure~\ref{subfig:CEMasonSeq} shows  the  SMPT rule  may
    lead to sub-optimal  solutions  when  no  setup  time  is  required  at  time
  $0$. Indeed, following the SMPT rule,  jobs of family $1$ have to be
  scheduled before  jobs of family $2$.  This leads to a  flow time of
  $198$. However,  schedule jobs of  family $2$ before jobs  of family
  $1$ leads to a better flow time, i.e. $181$.
  
  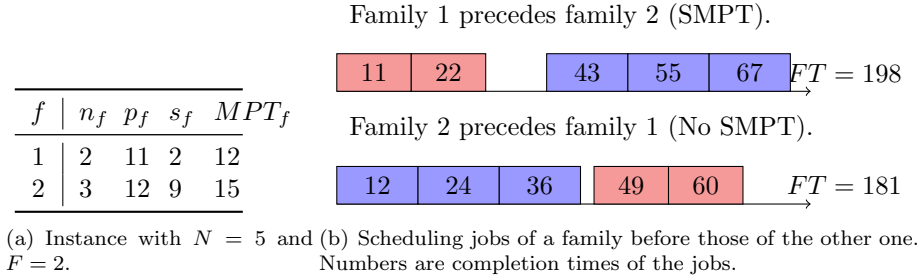
\begin{figure}[tp]
    \subfigure[Instance with $N=5$ and $F=2$.\label{subfig:CEMasonInst}]{
      \begin{tabularx}{0.25\linewidth}[b]{c|XXXX}
      \toprule
      $f$ & $n_f$ & $p_f$ & $s_f$ & $MPT_f$ \\
      \midrule
      1 & 2 & 11 & 2 & 12 \\
      2 & 3 & 12 & 9 & 15 \\
      \bottomrule
      \end{tabularx}
      \hspace{20pt}
    }
       \subfigure[Scheduling  jobs of  a  family before  those of  the
       other   one.    Numbers   are    completion   times    of   the
       jobs.\label{subfig:CEMasonSeq}]{ 

      \begin{tikzpicture}[xscale=0.09,yscale=0.5]
        \node (O) at (0,0) {};
        \draw[->] (O.center) -- (70,0);
        \node[right] at (0.5,2) {Family 1 precedes family 2 (SMPT).};
        \node at (75,0.5) {$FT = 198$};
        
        \draw[fill= red!90!black!40] (0,0) rectangle (11,1) node[midway] {$11$};
        \draw[fill= red!90!black!40] (11,0) rectangle (22,1) node[midway] {$22$};
        
        \draw[fill=blue!40] (31,0) rectangle (43,1) node[midway] {$43$};
        \draw[fill=blue!40] (43,0) rectangle (55,1) node[midway] {$55$};
        \draw[fill=blue!40] (55,0) rectangle (67,1) node[midway] {$67$};
        
        \begin{scope}[yshift = -1cm]
          \node (1) at (0,-2) {};
          \draw[->] (1.center) -- (70,-2);
          \node at (75,-1.5) {$FT = 181$};
          \node[right] at (0.5,0) {Family 2 precedes family 1 (No SMPT).};
          \draw[fill=blue!40] (0,-2) rectangle (12,-1) node[midway] {$12$};
          \draw[fill=blue!40] (12,-2) rectangle (24,-1) node[midway] {$24$};
          \draw[fill=blue!40] (24,-2) rectangle (36,-1) node[midway] {$36$};
          
          \draw[fill= red!90!black!40] (38,-2) rectangle (49,-1) node[midway] {$49$};
          \draw[fill= red!90!black!40] (49,-2) rectangle (60,-1) node[midway] {$60$};
        \end{scope}
      \end{tikzpicture}
    }
      \caption{Scheduling all blocks according to the SMPT rule is not optimal.}
       
      \label{fig:CEMason}
  \end{figure}
\end{example}

Actually, the  only reason why the  $SMPT$ rule does not  lead to an
optimal  solution  is  that  no  setup  time  is  required  at  time
$0$.  Therefore, in an optimal solution, all blocks except the first 
one  are  scheduled  according  to  the $SMPT$  rule.  That  is  the
statement  of   Theorem~\ref{th2}  for   which  a  proof   is  given
in~\cite{PMS2020}.

\begin{theorem}
  \label{th2}
  In an optimal sequence of the  problem, the blocks $2$ to $|\F|$ are
  ordered by  $SMPT$ (Shortest Mean  Processing Time).  That  is, if
  $1< i<j$ then $MPT(B_i) \le MPT(B_j)$.
\end{theorem}

The remaining of this section explains how these results are used to
define a polynomial time algorithm for sequencing jobs on a machine so
the flow time is minimized. This algorithm is called~\Seq in
the remaining of the paper.

Theorem~\ref{th4} states  that there exists an  optimal solution $S$
containing exactly  $|\F|$ blocks and  that each block  $B_i$ contains
all jobs  of family $f_i$.   Theorem~\ref{th2} states that  the blocks
$B_2$ to $B_{|\F|}$ are ordered by $SMPT$.  Finally, one only needs to
determine which family is processed in the very first block.

Algorithm \Seq takes as input the  set of jobs and starts by
grouping  them  in  blocks  and  sorting them  in  SMPT  order.   The
algorithm then computes the flow time of this schedule.  Each
block  is  then   successively  moved  to  the   first  position  (see
Figure~\ref{fig:moveFirst})  and the  new flow  time is  computed. The
solution returned by the algorithm  is therefore the one achieving the
best flow time.

\begin{figure}[htp]
  \begin{center}
    \subfigure[SMPT Sequence]{
      \begin{tikzpicture}[xscale=0.3,yscale=0.5]
        \node (O) at (0.5,0) {};
        \draw[->] (O.center) -- (21.5,0);

        \draw[fill=gray!30]   (0,0)  rectangle   (7,1)  node[midway]
        {$B_1, \dots, B_{f-1}$}; 
        \draw (7,0) rectangle (9,1) node[midway] {$s_f$};
        \draw[fill=blue!40]  (9,0)   rectangle  (12,1)  node[midway]
        {$B_{f}$}; 
        \draw (12,0) rectangle (14,1) node[midway] {$s_{f+1}$};
        \draw[fill=gray!30]  (14,0)  rectangle  (21,1)  node[midway]
        {$B_{f+1}, \dots, B_{|\F|}$}; 
      \end{tikzpicture}            }
    \subfigure[$B_f$ is moved in the first position.]{
      \begin{tikzpicture}[xscale=0.3,yscale=0.5]
        \node (O) at (0.5,0) {};
        \draw[->] (O.center) -- (21.5,0);
        
        \draw[fill=blue!40]   (0,0)  rectangle   (3,1)  node[midway]
        {$B_{f}$}; 
        \draw (3,0) rectangle (5,1) node[midway] {$s_1$};
        \draw[fill=gray!30]  (5,0)   rectangle  (12,1)  node[midway]
        {$B_1, \dots, B_{f-1}$}; 
        \draw (12,0) rectangle (14,1) node[midway] {$s_{f+1}$};
        \draw[fill=gray!30]  (14,0)  rectangle  (21,1)  node[midway]
        {$B_{f+1}, \dots, B_{|\F|}$}; 
      \end{tikzpicture}            
    }
    \caption{SMPT Sequence and Move Operation.}
    \label{fig:moveFirst}
  \end{center}
\end{figure}
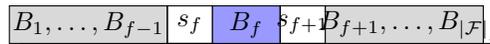
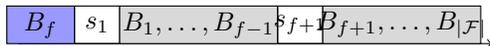

The  complexity  for   ordering  the  families  in   $SMPT$  order  is
$O(F\log{F})$.  The complexity of moving each block to the first 
position   and    computing   the    corresponding   flow    time   is
$O(F)$.  Indeed, there  is no  need to  re-compute the  entire flow
time.   The  difference  of  flow  time  coming  from  the  {\it  Move
  operation} can be computed in $O(1)$.  Hence, the complexity of
\Seq is $O(F \log{F})$.

\section{Filtering rules and algorithms}
\label{sec:filter}
This section is  dedicated to the cost-based filtering rules and algorithms derived from the
results of Section~\ref{sec:relaxation}.  They are
separated into three parts, each one corresponding to the filtering of one
variable: $flowtime_m$  (Section~\ref{sec:filtFT});  $nbJobs_{f,m}$
(Section~\ref{sec:filt_nJfm});               $nbJobs_m$
(Section~\ref{sec:nbjobs}).   Note  that   the   two  last   variables
constrained by the sum constraint~\ref{eq:nbJobsm1}. 

During the solving  of the problem, jobs are divided  for each machine
$m$   into  three   categories   based  on   the  interval   variables
$altJ_{j,m}$: each  job either is, or  can be, or is  not, assigned to
the machine. Note that the time  windows of the interval variables are
not considered in the relaxation. 
For a  machine $m$,  let $\setARE$ be  the set of  jobs for  which the
assignment  to machine  $m$  is  decided.  In  the  following of  this
section, an instance $\cal I$ is always considered with the set $\N^A=
\cup_{m \in \M} \setARE$ of jobs  already assigned to a machine. Thus,
an instance is denoted by $({\cal I}, \N^A )$.

Some  notations   are  now introduced.    For  a   variable  $x$,
$\underline{x}$ (resp. $\overline{x}$) represents the lower (resp. the 
upper) bound on  the variable $x$. Furthermore, let  $FT^*({\cal X})$ be
the flow time of the solution returned by the algorithm \Seq
applied on the set of jobs ${\cal X}$.

\subsection{Increasing the machine flow time}
\label{sec:filtFT}

The first  rule updates the  lower bound on the  $flowtime_m$ variable
and follows directly from Section~\ref{sec:relaxation}. The complexity
of this rule is $O(M \cdot F \cdot \log{F})$. 
\begin{filter}
  \label{filt:FT}
  $\forall m \in \M, \ flowtime_m \ge FT^*(\setARE) $
\end{filter}

\begin{proof}
  It is  sufficient to  notice that, for  a machine  $m$, \Seq
  gives a lower bound on the flow time. In particular, $FT^*(\setARE)$
  is a lower  bound on the flow  time of $m$ for  the instance $({\cal
    I}, \N^A)$.
\end{proof}




\begin{example}
  \label{ex:filtFT}
  Consider an instance with $3$ families. Their parameters are given
  by Table~\ref{tab:exdata}.  A specific  machine $m$ is  considered and
  set $\setARE$ is composed of one job of each family.  This instance
  and $\setARE$ is  used in all the example of this section.
\\
  Suppose $flowtime_m \in [  0 , 35]$. The output of  \Seq is given on
the top of  Figure~\ref{fig:seq1}.  Thus, the lower bound  on the flow
time can be updated to $2+7+13 = 22$, i.e. $flowtime_m\in [22, 35]$.
\\
Suppose now  that an extra job  of family $f_2$ is  added to $\setARE$
(on  the  bottom   of  Figure~\ref{fig:seq1}).  Thus,  $FT^*(\setARE)=
2+8+11+16=37$   and  $37   >   \overline{flowtime}_m=35$.  Thus,   the
assignment defined by $\setARE$ is infeasible.

\begin{figure}[tbp]
  \begin{center}
    \subfigure[Instance data.\label{tab:exdata}]{
    \begin{tabularx}{0.35\linewidth}[b]{c|CCC}
      \toprule
      $f$ & $n_f$ & $p_f$ & $s_f$ \nl
      \midrule
      1&3&2&5\nl
      2&3&3&3\nl
      3&4&4&1\nl
      \bottomrule
    \end{tabularx}
  }
    \subfigure[Illustration of \refR{FT} and \refR{nJfm}.\label{fig:seq1}]{
      \begin{minipage}[b]{0.6\linewidth}
        \begin{tikzpicture}[yscale = 0.5, xscale=0.35]
          \node (O) at (-0.5,0) {};
          \draw[->] (O.center) -- ( 14,0);
          \draw[fill =  red!90!black!40]  (0,0) rectangle (2,1) node[midway] {$f_1$};
          \draw[fill = green!80!black!40!]  (3,0) rectangle (7,1) node[midway] {$f_3$};
          \draw[fill = blue!40] (10,0) rectangle (13,1) node[midway] {$f_2$};
          
          \foreach \i in {0,1,2,...,13}
          { \draw (\i,0) -- (\i,-0.15);}
          \foreach \i in {0,5,10}
          { \draw (\i,0) -- (\i,-0.15) node[below] {\scriptsize \i};}
        \end{tikzpicture}
        \vspace{5pt}
        \begin{tikzpicture}[yscale = 0.5,xscale=0.35] 
          \node (O) at (-0.5,0) {};
          \draw[->] (O.center) -- ( 17,0);
          \draw[fill =  red!90!black!40]  (0,0) rectangle (2,1) node[midway] {$f_1$};
          \draw[fill = green!80!black!40!]  (12,0) rectangle (16,1) node[midway] {$f_3$};
          \draw[fill = blue!40] (5,0) rectangle (8,1) node[midway] {$f_2$};
          \draw[fill = blue!40] (8,0) rectangle (11,1) node[midway] {$f_2$};
          
          \foreach \i in {0,1,2,...,16}
          { \draw (\i,0) -- (\i,-0.15);}
          \foreach \i in {0,5,10,15}
          { \draw (\i,0) -- (\i,-0.15) node[below] {\scriptsize \i};}
          
        \end{tikzpicture}
      \end{minipage}
      }
      \caption{Illustration of $flowtime_m$ filtering (\refR{FT}).}
    \end{center}
  \end{figure}

          
          
          

\end{example}

Another rule can be  defined to filter $flowtime_m$.  This
rule is stronger than Rule~\ref{filt:FT} and is based on $\setARE$ and
$\underline{nbJobs}_m$.   Indeed,  $\underline{nbJobs}_m$  denotes the
minimum number of jobs  that has to be assigned to  $m$.  Thus, if one
can  find  the  $\underline{nbJobs}_m$  jobs  (including  the  one  of
$\setARE$) that leads  to the minimum flow time, it  will give a lower
bound on the flow time of machine $m$.
\\
Actually,  it  may  be  difficult  to  know  exactly  which  jobs  will
contribute the least to the flow time.  However, considering jobs in
$SPT$  order and  with $0$  setup time  gives a  valid lower  bound on
$flowtime_m$.  First, an example illustrating the filtering rule is
presented and then, the rule is formally given. 
\begin{example}
Consider  the instance  described in  Example~\ref{ex:filtFT}. Suppose
$\setARE$ is  composed of one job  of each family and  $flowtime_m \in
[22,60]$. Suppose also $\underline{nbJobs}_m 
= 6$.  Thus, $3$ extra jobs need to be assigned to $m$.
\\
Families in $SPT$ order are $\{f_1, f_2,f_3\}$ and the remaining number
of jobs in each family is $2,2,3$. Hence, the $3$ extra jobs are: $2$
jobs of $f_1$ and $1$ job of $f_2$.
\\
To make sure the lower bound on the flow time is valid, those jobs are
sequenced on  $m$ with no  setup time. In  Figure~\ref{fig:extFT}, $f_j$
denotes ``classical'' jobs of family $f_j$ while $f_j'$ denotes jobs of
family $f_j$ with no setup time. 
\\
Figure~\ref{fig:extFT} shows the results of \Seq on the set of
jobs   composed    of   $\setARE$    plus   the   $3$    extra   jobs,
i.e. $\underline{nbJobs}_m=6$.  Here,  $FT^* = 2+4+6+9+14+20=55$. Thus,
the lower  bound on  $flowtime_m$ can be  updated and  $flowtime_m \in
[55,60]$.

\vspace{0.3cm}
  \begin{center}
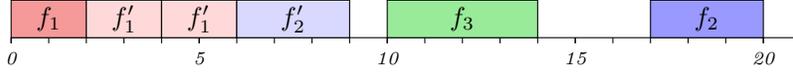

    \begin{tikzpicture}[scale=0.5]
      \node (O) at (0,0) {};
      \draw[->] (O.center) -- ( 21,0);
      \draw[fill = red!90!black!40]  (0,0) rectangle (2,1) node[midway] {$f_1$};
      \draw[fill = red!15] (2,0) rectangle (4,1) node[midway] {$f_1'$};
      \draw[fill = red!15]  (4,0) rectangle (6,1) node[midway] {$f_1'$};
      \draw[fill = blue!15]  (6,0) rectangle (9,1) node[midway] {$f_2'$};
      \draw[fill = green!80!black!40]  (10,0) rectangle (14,1) node[midway] {$f_3$};
      \draw[fill =  blue!40] (17,0) rectangle  (20,1) node[midway]
      {$f_2$};
      \foreach \i in {0,1,2,...,20}
      { \draw (\i,0) -- (\i,-0.15);}
      \foreach \i in {0,5,10,15,20}
      { \draw (\i,0) -- (\i,-0.15) node[below] {\scriptsize \i};}
    \end{tikzpicture}
  \end{center}
  \captionof{figure}{Illustration of $flowtime_m$ filtering (Rule~\ref{filt:extFT}).}
  \label{fig:extFT}

\vspace{0.5cm}

Note that, because Rule~\ref{filt:FT} does not take $\underline{nbJobs}_m$ into
account, it gives a lower bound of $22$ in this case.
\end{example}

Let   $\setExtFT$   denotes   the    set   composed   of   the   first
$\underline{nbJobs}_m - |\setARE|$ remaining  jobs in $SPT$ order with
setup time equal to $0$.

\begin{filter}
  \label{filt:extFT}
  $\forall m \in \M,\ flowtime_m \ge FT^*(\setARE \cup \setExtFT)$
\end{filter}

\begin{proof}
  First  note that  if  $ \setExtFT  = \emptyset$,  Rule~\ref{filt:FT}
  gives the result. Thus, suppose that $|\setExtFT| \ge 1$.
  By contradiction, suppose  $\exists m \in \M$ such  that $flowtime_m <
  FT^*(\setARE  \cup \setExtFT)$.  Thus, there  exists another  set of
  jobs $\setProofEFT$ with:
  \begin{equation}
    FT^*(\setARE \cup \setProofEFT) < FT^*(\setARE \cup \setExtFT) \label{eq:proofEFT}
  \end{equation}
   First, note that w.l.o.g. $|\setProofEFT| = |\setExtFT|$. Indeed, if
   $|\setProofEFT|   >   |\setExtFT|$,   we  can   remove   jobs   from
   $|\setProofEFT|$  without increasing  the  flow time.  Furthermore,
   w.l.o.g. we can consider that $\forall j \in \setProofEFT, \
  s_{f_j}=0$. Indeed, since setup times can only increase the flow time,
  thus inequality~\eqref{eq:proofEFT} is still verified.
  \\
  Let       $\overline{S}=       \{      \overline{j}_1,       \cdots,
  \overline{j}_{nbJobs_m} \}$ be the sequence returned by the \Seq
  algorithm on $\setARE \cup \setProofEFT$. Let also $\overline{j}_i$ be
  the  job  in  $\setProofEFT   \setminus  \setExtFT$  with  the  $SPT$.
  Finally,  let $j_k$  be the  job  with the  $SPT$ in  $ \in  \setExtFT
  \setminus  \setProofEFT$.    Thus,  since   $p_{f_{\overline{j}_i}}  >
  p_{f_{j_k}}$,  sequence $\overline{S}'  =  \{ \overline{j}_1,  \cdots,
  \overline{j}_{i-1},      j_k     ,      \overline{j}_{i+1},     \cdots
  \overline{j}_{nbJobs_m}   \}$   has   a   smaller   flow   time   than
  $\overline{S}$.

  Repeated  applications of  this operation  yield to  a contradiction
  with equation~\eqref{eq:proofEFT}.
\end{proof}

The complexity of Rule~\ref{filt:extFT} is $O ( M \cdot  F \cdot
\log{F})$. Indeed,  sorting families  in $SPT$ order  can be  done in
$O(F \cdot log{F})$. Creating the set $\setExtFT$ is done in $O(F)$ and
\Seq is applied in $O(F \cdot log{F})$ which gives a total complexity
of $O(F\cdot log{F} + M \cdot (F + F\cdot \log{F}))$.

\subsection{Reducing the maximum number of jobs of a family}
\label{sec:filt_nJfm}

The  idea of  the  filtering  rule presented  in  this  section is  as
follows.  For  a  family  $f$,  $\overline{nbJobs}_{f,m}$  define  the
maximum number  of jobs of  family $f$ that  can be scheduled  on $m$.
Thus, if adding those  $\overline{nbJobs}_{f,m}$ to $\setARE$ leads to
an  infeasibility, $\overline{nbJobs}_{f,m}$  can be  decreased by  at
least   $1$.   Let   denote  by   $\setNJfm$  the   set  composed   of
$\overline{nbJobs}_{f,m}$  jobs  of  family $f$  minus  those  already
present in $\setARE$.

\begin{filter}
  \label{filt:nJfm}
  If  $\exists(f,m) \in  \F \times  \M$ such  that $FT^*(\setARE  \cup
  \setNJfm)   >   \overline{flowtime}_m$,   then   $nbJobs_{f,m}   \le
  \overline{nbJobs}_{f,m} - 1 $ 
\end{filter}

\begin{proof}
  Suppose that  for a family  $f$ and a machine  $m$, we have  a valid
  assignment such that $ FT^*(\setARE \cup \setNJfm)
  >      \overline{flowtime}_m$     and      $     nbJobs_{f,m}      =
  \overline{nbJobs}_{f,m}$.  By Rule~\ref{filt:FT}, the assignment
  is infeasible which is a contradiction.  
\end{proof}

\begin{example}
Let us  consider the  instance defined by  example~\ref{ex:filtFT}. In
the first  part of this example,  $\setARE$ is composed of  one job of
each   family    and   $flowtime_m   \in   [22,35]$.    Suppose   that
$\overline{nbJobs}_{f_2,m}  =   2$.  Thus,  $\setARE   \cup  \setNJfm$
contains one  job of family  $f_1$ and $f_3$  and two jobs  of family
$f_2$.  The bottom part of figure~\ref{fig:seq1} shows that $ FT^*(\setARE \cup
\setNJfm)    =   37    >    \overline{flowtime}_m   =   35$.    Thus,
$\overline{nbJobs}_{f_2,m} < 2$. 
\end{example}

The complexity of Rule~\ref{filt:nJfm} is $O ( M \cdot F \cdot \log{F}
+ M \cdot F^2 \cdot \log |{\cal D}_{nbJobs_{f,m}}^{MAX}|)$ with $|{\cal
D}_{nbJobs_{f,m}}^{MAX}|$ the maximum size  of the domain of variables
$nbJobs_{f,m}$. Indeed,  the first  part corresponds to  the complexity
for applying the \Seq algorithm on all machine. This algorithm
only  need to  be applied  once since  each time  we remove  jobs from
$\setNJfm$,  $FT^*$  can  be  updated in  $O(F)$.   Indeed,  only  the
position of the family in the sequence has to be updated. Thus, the
second part corresponds to the updating  of $FT^*$ for each family and
each machine. By  proceeding by dichotomy, this update has  to be done
at  most  $\log  |{\cal  D}_{nbJobs_{f,m}}^{MAX}|$  times.  Thus,  the
complexity of Rule~\ref{filt:nJfm} is $O (  M \cdot F \cdot ( \log{F}+
F \cdot \log |{\cal D}_{nbJobs_{f,m}}^{MAX}|))$.

\subsection{Reducing the maximum number of jobs on a machine}
\label{sec:nbjobs}

The  idea  behind Rule~\ref{filt:extFT}  can  be  used to  reduce  the
maximum number  of jobs on  machine $m$.   Indeed, for a  machine $m$,
$\overline{nbJobs}_m$  is  the maximum  number  of  jobs that  can  be
scheduled  to   $m$.  Thus,  if   it  is  not  possible   to  schedule
$\overline{nbJobs}_m$  on $m$  without exceeding  the flow  time, then
$\overline{nbJobs}_m$ can be decreased.

The extra jobs that will be assigned on $m$ must be decided. Note that
those  jobs   must  give  a   lower  bound   on  the  flow   time  for
$\overline{nbJobs}_m$ jobs with the pre-assignment defined by 
$\setARE$.  Thus, jobs can be considered  in $SPT$ order with no setup
time.  Before  giving   the  exact  filtering  rule,   an  example  is
described. 

\begin{example}
Consider  the instance  described in  Example~\ref{ex:filtFT}. In  the
first part of  this example, $\setARE$  is composed of one job of each
family and  $flowtime_m \in  [22,60]$. Suppose  $\overline{nbJobs}_m =
7$.  Thus, the $4$ extra jobs assigned  to $m$ are: $2$ jobs of $f_1$
and  $2$ jobs  of  $f_2$. Figure~\ref{fig:nJm}  shows  the results  of
\Seq on  the set of  jobs composed  of $\setARE$ plus  the $4$
extra jobs.   Here, $FT^*  = 2+4+6+11+14+17+23 =  77$ which  is greater
than $\overline{flowtime}_m=60$. Thus,  $\overline{nbJobs}_m$ cannot be
equal to $7$ and can be filtered. 

\vspace{0.3cm}
\begin{center}
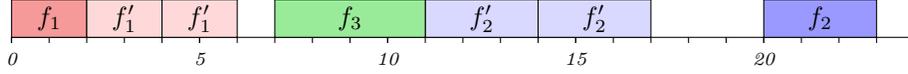

      \begin{tikzpicture}[scale=0.5]
        \node (O) at (0,0) {};
        \draw[->] (O.center) -- ( 24,0);
        \draw[fill =red!90!black!40]  (0,0) rectangle (2,1) node[midway] {$f_1$};
        \draw[fill = red!15]  (2,0) rectangle (4,1) node[midway] {$f_1'$};
        \draw[fill = red!15]  (4,0) rectangle (6,1) node[midway] {$f_1'$};
        \draw[fill = blue!15]  (11,0) rectangle (14,1) node[midway] {$f_2'$};
        \draw[fill = blue!15]  (14,0) rectangle (17,1) node[midway] {$f_2'$};
        \draw[fill = green!80!black!40]  (7,0) rectangle (11,1) node[midway] {$f_3$};
        \draw[fill =  blue!40] (20,0) rectangle  (23,1) node[midway]
        {$f_2$};
        
        \foreach \i in {0,1,2,...,23}
        { \draw (\i,0) -- (\i,-0.15);}
        \foreach \i in {0,5,10,15,20}
        { \draw (\i,0) -- (\i,-0.15) node[below] {\scriptsize \i};}

      \end{tikzpicture}
\end{center}
\captionof{figure}{Illustration of Rule~\ref{filt:nJm}.}
\label{fig:nJm}

\vspace{0.5cm}
\end{example}

Let   $\setNJm$   denotes   the   set  composed   of   the   first   $
\overline{nbJobs}_m -  |\setARE|$ jobs  in $SPT$  order with  setup time
equal to $0$.

\begin{filter}
  \label{filt:nJm}
  If $\exists  m \in  \M$  such that $  FT^*(\setARE \cup
  \setNJm) > \overline{flowtime}_m$, then $nbJobs_{m} \le \overline{nbJobs}_{m} - 1 $
\end{filter}

\begin{proof}
 The arguments are similar to those for~\refR{extFT}.
\end{proof}

The complexity of Rule~\ref{filt:nJm} is $O ( M \cdot F \cdot \log{F}
+ M \cdot F^2 \cdot |{\cal D}_{nbJobs_{m}}^{MAX}|)$ with $|{\cal
D}_{nbJobs_{m}}^{MAX}|$ the maximum size  of the domain of variables
$nbJobs_{m}$.  Indeed,  as  for  Rule~\ref{filt:nJfm},  the  algorithm
\Seq only needs to be applied  once and then can be updated in
$O(F)$.  For each machine and each  family, this update has to be done
at most  $|{\cal D}_{nbJobs_{m}}^{MAX}|$ times.   Indeed, proceeding
by dichotomy here implies that $FT^*$  cannot be updated but has to be
re-computed each time. Thus, the complexity
of Rule~\ref{filt:nJm}  is $O (  M \cdot F  \cdot ( \log{F}+  F \cdot
|{\cal D}_{nbJobs_{m}}^{MAX}|))$.  

\section{Experimental Results}
\label{sec:expe}

This section starts with the  presentation of the general framework of
the experiments in~\ref{sec:framework}.   Following the framework, the
filtering rules are  evaluated in Section~\ref{sec:eval-rules}.  Then,
the   model   is   compared   to    those   of   the   literature   in
Section~\ref{sec:cmp-models}.  Last,  a brief sensitivity  analysis is
given in Section~\ref{sec:sens-ana}.


\subsection{Framework}
\label{sec:framework}

The experiment framework is defined so the following questions are addressed.
\begin{compactenum}[\normalfont \it Q1.]
\item Which filtering rule is efficient? Are filtering rules complementary?
\item Which model of the literature is the most efficient?
\item What is the impact of the heuristics? Of the bi-objective aggregation?  
\end{compactenum}
To address these questions, the following metrics are analyzed: number
of feasible solutions, number of proven optimums, upper bound quality;
solving times; number of fails (for CP only).  

The benchmark instances used to perform our experiments are extracted
from~\cite{Nattaf2018C}. In this paper, 19 instance sets are generated
with different number of jobs ($N$), machines ($M$), family ($F$) and
qualification schemes.  Each of the instance sets is a group of 30
instances.  There is  a total  of 570  feasible instances  with $N=20$
(180), $N=30$  (180), $N=40$  (30), $N=50$  (30), $N=60$  (60), $N=70$
(90).  

The naming scheme for the different solving algorithms is described in
Table~\ref{tab:encoding}.  The first letter represents the model where
\IP  ,  \CP  and  \CPN  denotes respectively  the  \IP  and  CP  model
of~\cite{Nattaf2019C},    and    the     CP    model    detailed    in
Section~\ref{sec:CP}.  The models are implemented using IBM ILOG CPLEX
Optimization Studio 12.10~\cite{Ilog12100}.  That is CPLEX for the \IP
and CP Optimizer for CP models. 
The second letter indicates whether two heuristics are executed to find
solutions which are used as a basis for the models. 
These  heuristics are  called {\it  Scheduling Centric  Heuristic} and
{\it Qualification Centric Heuristic}~\cite{Nattaf2018C}. 
The goal of the first heuristic is to minimize the flow time while the
second one tries to minimize the number of disqualifications. 
The third letter indicates the filtering rules that are activated for the \CPN.
\refR{extFT} is  used for  the letter \texttt{L}  because it  has been
shown more efficient than \refR{FT} in preliminary experiments. 
The  fourth  letter indicates  the  bi-objective  aggregation method:
lexicographic   optimization;    linearization   of   lexicographic
optimization.  
The last letter  indicates the objective priority.  Here, the priority
is given to the flow time in all experiments because the
cost-based filtering rules concern the flow time objective.

\begin{center}
  \begin{tabularx}{0.9\linewidth}{tl*{5}{|tl}}
    \toprule
    \multicolumn{2}{c}{Model}   &    \multicolumn{2}{c}{Heuristic}   &
    \multicolumn{4}{c}{Filtering                rule}                &
    \multicolumn{2}{c}{Bi-objective} & \multicolumn{2}{c}{Priority} \\ 
    \cmidrule(r){1-2} \cmidrule(lr){3-4} \cmidrule(lr){5-8} \cmidrule(lr){9-10} \cmidrule(l){11-12}
    I & \IP & \_ & None & L  & \refR{extFT} & \_ & None & S & Weighted
    sum & F & Flow time \\
    O & \CP       & H  & All  & F & \refR{nJfm}  & A  & All  & L & Lexicographic & Q & Disqualifications \\
    N & \CPN & & & M & \refR{nJm} & & & & \\
    \bottomrule
  \end{tabularx}
  \captionof{table}{Algorithms Encoding.}
  \label{tab:encoding}
\end{center}

  All the experiments were led on a Dell computer with 256 GB of RAM and
4 Intel  E7-4870 2.40  GHz processors running  CentOS Linux  release 7
(each processor has  10 cores).  The time limit for  each run is $300$
seconds.

\subsection{Evaluation of the filtering rules}
\label{sec:eval-rules}

In  this  section, the  efficiency  and  complementary nature  of  the
proposed  filtering  rules  are  investigated.  In  other  words,  the
algorithms \texttt{N\_*LF}  are studied.  To this end,  the heuristics
are   not  used   to  initialize   the  solvers.    The  lexicographic
optimization  is  used since  it  has  been  shown more  efficient  in
preliminary experiments.

\begin{figure}[!htbp]
  \centering
  \subfigure[Cumulative number of optima.\label{fig:rules-opt}]{%
    \includegraphics[width=0.65\textwidth]{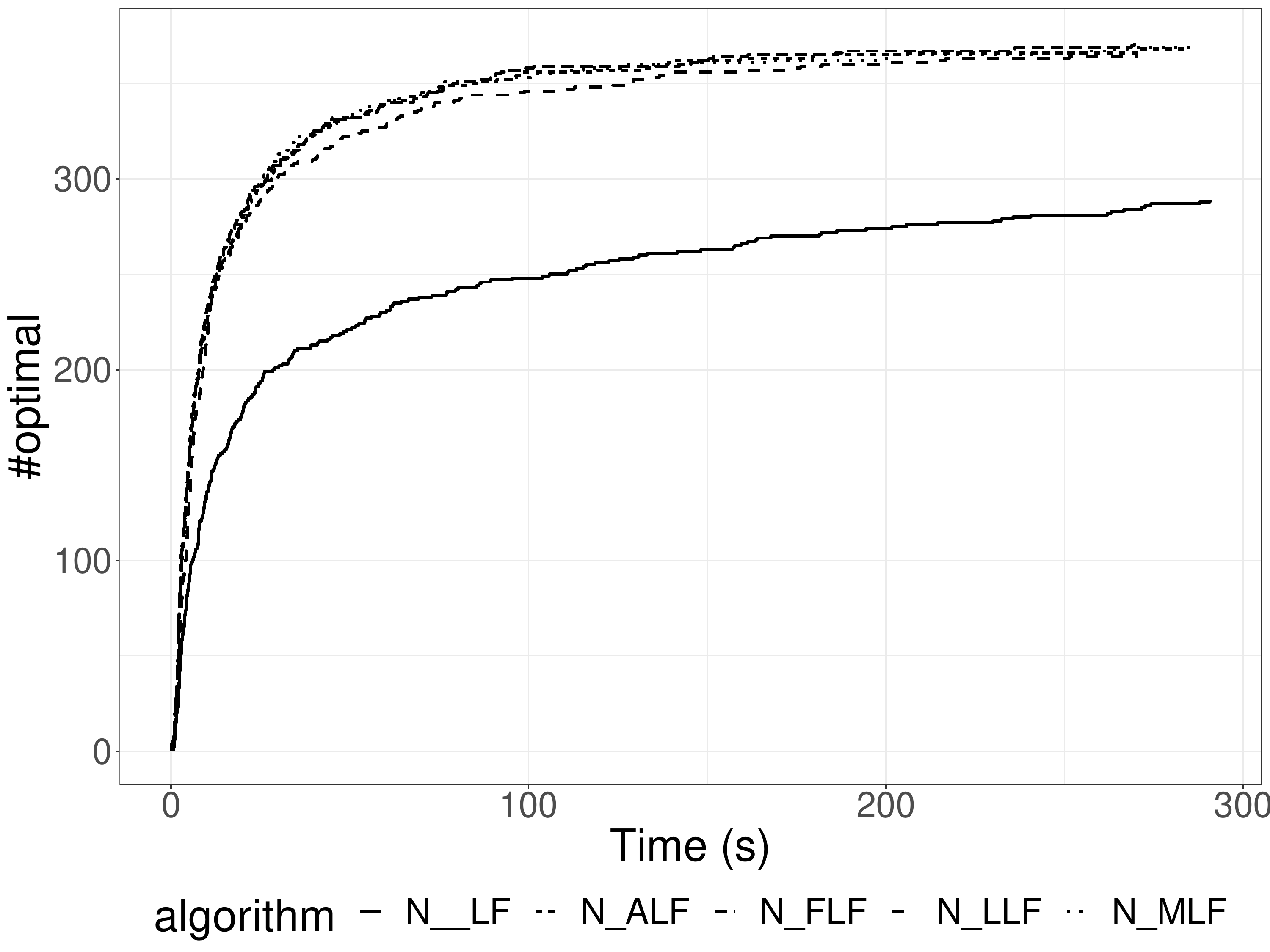}
  }
  \subfigure[Borda ranking.\label{fig:rules-rank}]{
    \hspace{10pt}%
    \begin{tabular}[b]{tr}
      \toprule
      \multicolumn{1}{c}{Algorithm} & \multicolumn{1}{c}{Score} \\
      \midrule
      N\_MLF & 1612 \\ 
      N\_ALF & 1644 \\
      N\_FLF & 1650 \\
      N\_LLF & 1712 \\
      N\_\_LF & 1932 \\
      \bottomrule
      \\ \\ \\
    \end{tabular}
  }  
  \caption{Evaluation of filtering rules.}
  \label{fig:rules}
\end{figure}

First, all  algorithms find feasible  solutions for more than  99\% of
the  instances.  Then,  for each  algorithm, the  number of  instances
solved   optimally  is   drawn  as   a   function  of   the  time   in
Figure~\ref{fig:rules-opt}. The leftmost curve  is the fastest whereas
the  topmost  curve proves  the  more  optima.  Clearly,  compared  to
\texttt{N\_\_LF}, the filtering rules  accelerates the proof and allow
the optimal solving of around eighty more instances.  One can notice that
the  advanced   filtering  rules   (\texttt{N\_FLF},  \texttt{N\_MLF},
\texttt{N\_ALF}),  also slightly  improves the  proof compared  to the
simple update of the flow time lower bound (\texttt{N\_LLF}). Here, the
advanced filtering rules are indistinguishable.

Table~\ref{fig:rules-rank} ranks  the filtering  rules according  to a
scoring    procedure    based    on    the    Borda    count    voting
system~\cite{brams.fishburn-02}.   In this  procedure, each  benchmark
instance  is treated  like a  voter  who ranks  the algorithms.   Each
algorithm  scores  points based  on  their  fractional ranking:  the
algorithms compare equal receive the same ranking number, which is the
mean of what  they would have under ordinal rankings.   Here, the rank
only depends on the answer: the solution status (optimum, satisfiable,
and unknown); and  then the  objective value. So,  the lower  is the
score, the  better is the  algorithm.  Once again, advanced  rules are
really close, and slightly above  the simple lower bound update.  They
clearly outperform the algorithm without cost-based filtering.

\subsection{Comparison to the literature}
\label{sec:cmp-models}

In this section, the best possible  algorithms using the \IP, \CP, and
\CPN are  compared. Here,  the heuristics are  used to  initialize the
solvers,   and   thus the   algorithms   \texttt{IH\_SF},
\texttt{OH\_LF}, \texttt{NHALF} are studied.

\begin{figure}[!htbp]
  \centering
  \subfigure[Cumulative number of optima.\label{fig:models-opt}]{%
    \includegraphics[width=0.65\textwidth]{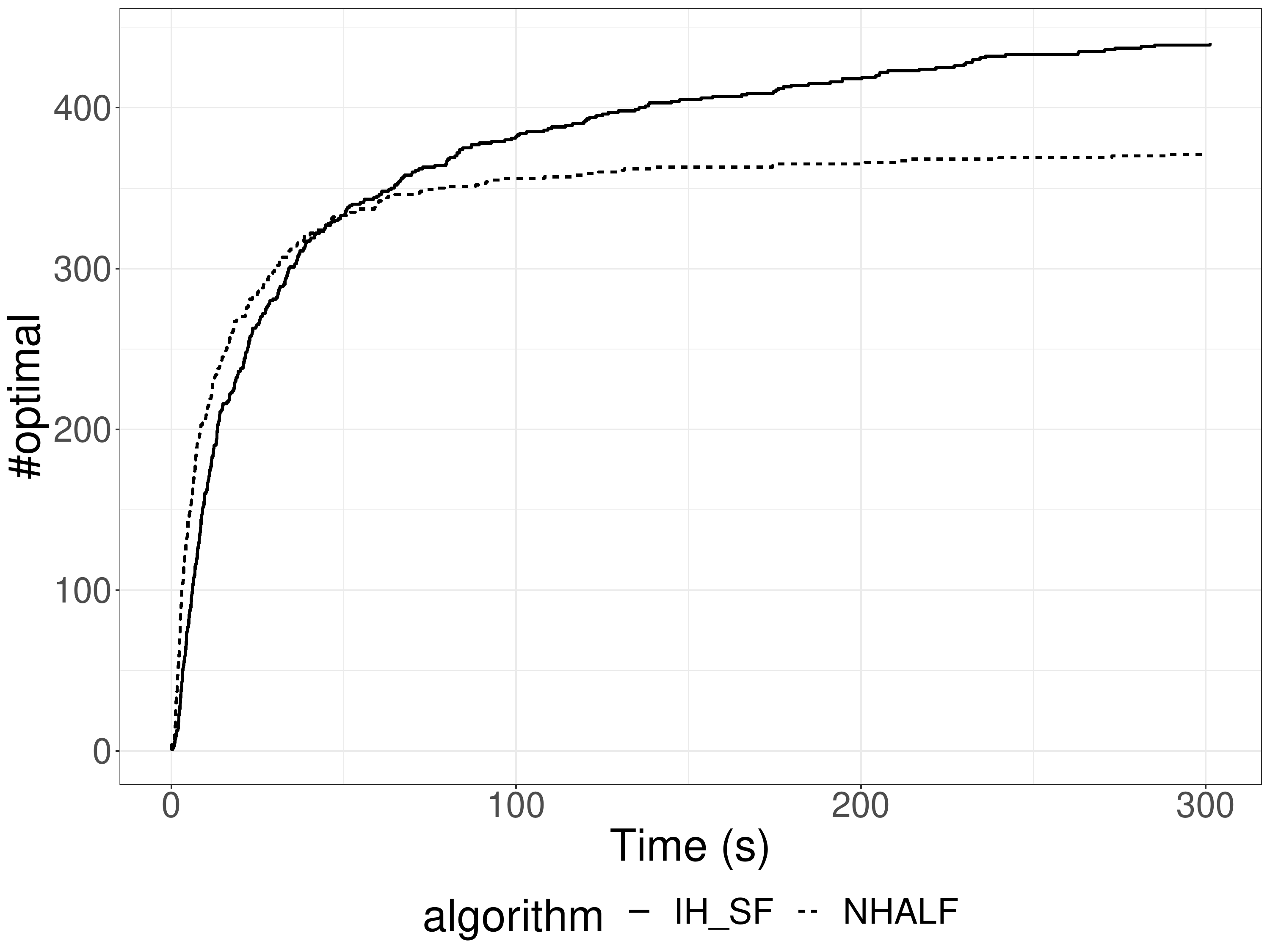}
  }
  \subfigure[Borda ranking.\label{fig:models-rank}]{
    \hspace{10pt}%
    \begin{tabular}[b]{tr}

      \toprule
      \multicolumn{1}{c}{Algorithm} & \multicolumn{1}{c}{Score} \\
      \midrule
      IH\_SF &  942.5 \\
      NHALF &  951.5 \\
      OH\_LF & 1526.0 \\
      \bottomrule
      \\ \\ \\
    \end{tabular}
  }  
  \caption{Comparison of the models.}
  \label{fig:models}
\end{figure}

First,  the  models  exhibit  very different  behaviors  in  terms  of
satisfiability and  optimality. Then,  for each  model, the  number of
instances  solved optimally  is drawn  as a  function of  the time  in
Figure~\ref{fig:models-opt}.  The  \CP is not visible  because it does
not prove any optimum.  The \CPN  is faster than the  \IP, but proves
less optima (around seventy).

Table~\ref{fig:models-rank} ranks  the models  rules according  to the
Borda  count  voting  system  (see  Section~\ref{sec:eval-rules}).  It
confirms  the  poor efficiency  of  the  \CP,  but brings  closer  the
performance of \CPN and \IP. 
Table~\ref{fig:ctab1} explains  these close scores. \IP
proves  $67=70-3$  more optima  than  \CPN,  but \CPN  finds  feasible
solutions for 32 more instances. 
\\
To conclude, the \CPN is competitive with the \IP and they offer
orthogonal performance since the \IP is more efficient for proving the
optimality  whereas the  \CPN is  more efficient  for quickly  finding
feasible solutions.

\subsection{Sensitivity analysis}
\label{sec:sens-ana}
Here,   the  impact   of  the  heuristics  and   of  the  bi-objective
optimization on the  efficiency of the models is analyzed.  The \CP is
excluded since it is clearly dominated by the two other models.

The   heuristics   have   a   little   impact   on   the   \CPN   (see
Table~\ref{fig:ctab1}). The solution status stay identical $98.5\%$ of
the time and the solving times  remain in the same order of magnitude.
However,  Table~\ref{fig:ctab2} shows  the significant
impact of the heuristics on the  \IP where the answer for 50 instances
becomes satisfiable.  For  both models, the heuristics do  not help to
prove more optima.
\\
Table~\ref{fig:ctab3} shows the significant impact of using the lexicographic
optimization instead of the weighted sum method on
the \CP.  Indeed, 69 instances with unknown status when using the weighted
sum    method    become    satisfiable   using    the    lexicographic
optimization. Note that the lexicographic optimization is not available for the \IP.

\begin{figure}[htb]
  \centering
  \subfigure[Two best algorithms.\label{fig:ctab1}]{
    \begin{tabular}[b]{l|rrr}
      \toprule
      \multicolumn{1}{t|}{NHALF}& \multicolumn{3}{t}{IH\_SF} \\
      & \multicolumn{1}{c}{OPT} & \multicolumn{1}{c}{SAT} & \multicolumn{1}{c}{UNK} \\
      \midrule
      OPT & 369  &  3 &   0 \\
      SAT &  70  & 94 &  32 \\
      UNK &   1  &  0 &   1 \\
      \bottomrule
    \end{tabular}
  }\hfill
    \subfigure[Heuristics impact.\label{fig:ctab2}]{
    \begin{tabular}[b]{l|rrr}
      \toprule
      \multicolumn{1}{t|}{I\_\_SF}& \multicolumn{3}{t}{IH\_SF} \\
      & \multicolumn{1}{c}{OPT} & \multicolumn{1}{c}{SAT} & \multicolumn{1}{c}{UNK} \\
      \midrule
      OPT & 433 &   3 &   0 \\
      SAT &   6 &  44 &   0 \\
      UNK &   1 &  50 &  33 \\
      \bottomrule

    \end{tabular}
  }\hfill
  \subfigure[Aggregation impact. \label{fig:ctab3}]{
    \begin{tabular}[b]{l|rrr}
      \toprule
      \multicolumn{1}{t|}{N\_ALF}& \multicolumn{3}{t}{N\_ASF} \\
      & \multicolumn{1}{c}{OPT} & \multicolumn{1}{c}{SAT} & \multicolumn{1}{c}{UNK} \\
      \midrule
      OPT & 360 &   8 &   1 \\
      SAT &   4 & 125 &  69 \\
      UNK &   0 &   0 &   3 \\
      \bottomrule

    \end{tabular}
  }
  \caption{Contingency table  of the solution status  between pairs of
    algorithms.} 
  \label{fig:ctab}
\end{figure}

\section{Conclusion}

In  this  paper, cost-based  domain  filtering  has been  successfully
applied to  an efficient  constraint programming model  for scheduling
problems with setup on parallel machines. 
The  filtering rules  derive from  a polynomial-time  algorithm which
minimize the flow time for a single machine relaxation. 
Experimental  results   have  shown  the  rules   efficiency  and  the
competitiveness of the overall algorithm. 
The first perspective  is to tackle larger  industrial instances since
CP relies on its ability at finding feasible solutions. 
The second perspective is to pay  more attention to the propagation of
the flow time based on what has been done for the makespan.

\bibliographystyle{plainnat}
\bibliography{biblioAPC}
\end{document}